\RequirePackage{amssymb}
\RequirePackage{amsmath}
\documentclass[runningheads]{llncs}
\usepackage{graphicx}
\newcommand{\nonprob}{\mathrm{non\text{-}prob}}

\makeatletter
\def\blfootnote{\xdef\@thefnmark{}\@footnotetext}
\makeatother

\begin{document}
\title{Causal Modeling with Probabilistic Simulation Models\thanks{Thanks to Thomas Icard for helpful discussions. The author was supported by the Sudhakar and Sumithra Ravi Family Graduate Fellowship in the School of Engineering at Stanford University for this work.}}
\author{Duligur Ibeling%\inst{1}%\orcidID{0000-1111-2222-3333}
}
\authorrunning{D. Ibeling}
\institute{Stanford University, Stanford, CA, USA\\
\email{duligur@stanford.edu}}
\maketitle              % typeset the header of the contribution
\begin{abstract}
%% The setting of (Ibeling and Icard submitted 2018) is ex-
%% tended to the case of probabilistic simulation models.
%% The special case of a.s.-terminating models is discussed.
%% It is shown that there is a natural definition of probabil-
%% ity on the language that induces reasonable conditional
%% probabilities, allowing for the expression of facts about
%% conditional simulation and counterfactuals by integrat-
%% ing interventions and new observations in the same for-
%% mula, as in (Pearl 2009). Conditional probabilities are
%% discussed in some detail. A logic for reasoning about
%% linear (up to a constant) inequalities on probabilities is
%% given and soundness and completeness are shown. NP-
%% completeness of the satisfiability problem for varphis logic
%% is also proven.
\blfootnote{In \emph{Proceedings of the 5th Workshop on Probabilistic Logic Programming} (PLP 2018).}
Recent authors have proposed analyzing conditional reasoning through a notion of intervention on a simulation program, and have found a sound and complete axiomatization of the logic of conditionals in this setting. Here we extend this setting to the case of probabilistic simulation models. We give a natural definition of probability on formulas of the conditional language, allowing for the expression of counterfactuals, and prove foundational results about this definition. We also find an axiomatization for reasoning about linear inequalities involving probabilities in this setting. We prove soundness, completeness, and $\mathsf{NP}$-completeness of the satisfiability problem for this logic.
%  \cite{ibeling} have proposed analyzing conditional reasoning
%  through a notion of intervention on a simulation program, finding
%  a sound and complete axiomatization of the logic of conditionals
%  in this setting.
%  Here we extend this setting to the important case of probabilistic simulation models.
%  We give a natural definition of probability on
%  formulas of the conditional language, allowing for the expression of counterfactuals, and prove
%  foundational results about this definition.
%  We also find an axiomatization for reasoning about linear inequalities
%  involving probabilities in this setting, in the same vein as the (purely propositional) logic
%  of \cite{Fagin}.
%  We prove soundness, completeness, and $\mathsf{NP}$-completeness of the satisfiability problem
%  for this logic.
% The abstract should briefly summarize the contents of the paper in

\end{abstract}
\section{Introduction}
Accounts of subjunctive conditionals based on internal \emph{causal models}
offer an alternative to approaches based on ranking possible worlds by similarity \cite{lewis73}.
One might, e.g., employ \emph{structural
equation models} (SEMs), i.e. systems of equations connecting the values of relevant variables,
as the causal model;
the semantics of conditionals are then based on a precise notion of \emph{intervention} on
the SEM \cite{Pearl2009}.
%\cite{Halpern2000} has given a sound and complete
%axiomatization of the logic of conditionals in this setting.
Recently, some authors \cite{AC17,BLOG,Goodman2014,DBLP:books/cu/p/FreerRT14,chater2013programs}
have proposed using arbitrary programs, rather than systems of equations, as causal models.
This approach emphasizes the procedural nature of many internal causal simulations over the purely declarative SEMs.
% This approach is already implicit in, e.g., generative models with probabilistic programming.

%\footnote{varphis is not strictly an extension, since \cite{Pearl2009} does not require SEMs to use only
%computable functions, or any other conditions on their domain.
%However, this is a reasonable restriction in practice and once one makes
%it the simulation models become an extension of SEMs.}
It is possible to define precisely this idea of programs as causal models and to generalize
the idea of intervention from SEMs to programs \cite{AC17}.
It is also possible to give a sound and complete 
logic of conditionals in this setting \cite{ibeling}.
However, these preliminary results
have not fully explored the very important case---from, e.g.,
the Bayesian Logic modeling language
\cite{BLOG} and implicit in the use of probabilistic programs as cognitive models \cite{DBLP:books/cu/p/FreerRT14}---of conditionals in a probabilistic setting, via using stochastic programs
as the underlying causal model.
In the present contribution we will establish
foundational definitions and logical results for this setting,
thus
extending the causal simulation framework to probabilistic simulation programs.
Probabilities over a causal modeling language are defined and results showing that they may
actually be interpreted as probabilities are given.
%These probabilities lead to natural conditional probabilities varphich reflect the conditional simulation
%framework
%\cite{DBLP:books/cu/p/FreerRT14} and give a natural version of the $\mathrm{do}$-language
%of \cite{Pearl2009} for the simulation model case. 
The probabilities are used to give the semantics of a language for probabilistic reasoning, for which
an axiomatization is given.
The language and axiomatization are extensions of an analogous probabilistic language
considered for the purely propositional case by \cite{Fagin}.
Soundness and completeness of the axiom system is proven,
and the satisfiability problem is found to be $\mathsf{NP}$-complete.

\section{Probabilistic Simulation Models and the Logical Language}
\subsection{Simulation Models}
We work toward the definition of a language $\mathcal{L}$ for expressing
probabilities involving probabilistic simulation models.
Probabilistic simulation models extend the non-probabilistic\footnote{
			The use of ``non-probabilistic'' rather than ``deterministic'' is intended to prevent confusion of
			the probabilistic/non-probabilistic distinction with the deterministic Turing
			machine/non-deterministic Turing machine distinction.
			The former distinction is about the presence of a source of randomness while the latter is about the number of possible halting executions.} causal simulation models of
\cite{AC17,ibeling}.
% TODO: easily seen to be translatable into
% Turing machine code *(?)
Formally, a \emph{non-probabilistic simulation model} is a % TODO: deterministic?
Turing machine\footnote{
\cite{ibeling} does not require these machines to be deterministic, and isolates
an additional logical principle that is valid when the machines are deterministic. However here
we will suppose ``non-probabilistic simulation model'' always refers to one whose Turing machine
is deterministic. This definition is more useful for comparison with the probabilistic case, in which
all underlying machines are deterministic.
},
and a \emph{probabilistic simulation model} is a probabilistic Turing machine, i.e., a deterministic
Turing machine (that of course still has a read-write memory tape)
given read access to a random bit tape whose squares represent the results
of independent fair coin flips.
The use of Turing machines is meant to allow for complete generality and encompasses, e.g., both
logic programming and imperative programming.
We sometimes use intuitive pseudocode in describing simulation models; such pseudocode is
readily convertible to Turing machine code.
% TODO:add

We suppose that simulation models are run initially from an empty tape.\footnote{
	\cite{ibeling} also includes an initial input tape in the definition of the model. This difference is inconsequential.
	% TODO: generative?
	}
As a simulation model runs,
it reads and writes the values of binary variables on its tape squares.
Eventually, the model either halts with some resultant tape,
or does not halt, depending on the results of the coin flips the model performs
in the course of its simulation. Every probabilistic simulation model
thus induces a distribution on these possible outcomes.
We are interested not only in these outcomes, but also in the
dynamics and counterfactual
information embodied in the model. That is, we are interested in what
\emph{would} happen were we to hold the values of the tape square variables fixed in a particular way
that counterfactually differs from the actual values the squares take on---in the
distribution over outcomes that results under a particular \emph{intervention}:
\begin{definition}[Intervention \cite{AC17}]
Let $S$ be a specification of binary values for
a finite number of tape squares:
$S = \{x_i\}_{i \in I}$ for
a finite index set
$I \subseteq \mathbb{N}$.
Then the
intervention $\mathcal{I}_S$ is a computable function
from Turing machines to 
Turing machines specified in the following way.
Given a machine $\mathsf{T}$, the intervened machine $\mathcal{I}_S(\mathsf{T})$
does the same thing as $\mathsf{T}$ but holds the variables in $S$ to their fixed
values specified by $S$ throughout the run.
That is,
$\mathcal{I}_S(\mathsf{T})$ first writes $x_i$ to square $i$ for all $i \in I$, then runs $\mathsf{T}$
while ignoring any writes to any of the squares whose indices are in $I$.
% $X_i$ to $x_i$ throughout the execution of $\mathsf{T}$.
% That is, $\mathcal{I}(\mathsf{T})$ first sets each $X_i$ to $x_i$, then runs $\mathsf{T}$
% while ignoring any write to any $X_i$.
\end{definition}

Suppose one fixes the entire random bit tape to some particular sequence in $\{0, 1\}^{\infty}$.
Then the counterfactual, as well as actual,
behavior of a probabilistic simulation model is completely non-probabilistic.
We define first a basic language
that allows us to express facts about such behavior.
Then we will define the probability that a given
probabilistic simulation model satisfies a formula of this basic language.
Our final language $\mathcal{L}$ uses these probabilities---it
thus expresses facts about the \emph{probabilities} that counterfactual properties
hold.
In all logical expressions we help ourselves to these standard notational
conventions:
$\alpha \rightarrow \beta$ abbreviates
$\lnot \alpha \lor \beta$,
and $\alpha \leftrightarrow \beta$ denotes
$(\alpha \rightarrow \beta) \land (\beta \rightarrow \alpha)$.
%Also, for any satisfaction relation $\models$ we will define, $\models \varphi$
%abbreviates the fact that $\varphi$ is valid, with respect to some class of models
%that will be made clear by the context.

\subsection{The Basic Language} 
\subsubsection{Syntax}
The basic, non-probabilistic language $\mathcal{L}_{\nonprob}$ is a propositional language over conditionals. Formally:
\begin{definition}
Let $X$ be a set of atoms $\{X_1, X_2, X_3, \dots\}$ representing the values of the memory tape variables
and let $\mathcal{L}_{\mathrm{prop}}$ be the propositional language formed by closing $X$
off under conjunction, disjunction, and negation.

	Let the intervention specification language $\mathcal{L}_{\mathrm{int}} \subset \mathcal{L}_{\mathrm{prop}}$
be the language of
purely conjunctive, ordered formulas of unique literals,\footnote{
The point being that such formulas are in one-to-one correspondence with specifications of interventions,
i.e., finite lists of variables along with the values each is to be held fixed to.
} i.e., formulas of the form $l_{i_1} \land \dots \land l_{i_n}$ for some $n \ge 0$, where $i_j < i_{j+1}$ and each $l_{i_j}$ is either $X_{i_j}$ or $\lnot X_{i_j}$.
$\top$ abbreviates the ``empty intervention'' formula with $n = 0$.
	Let $\mathcal{L}_{\mathrm{cond}}$ be the conditional language of formulas
of the form $\langle\alpha\rangle \beta$ for $\alpha \in \mathcal{L}_{\mathrm{int}}, \beta \in \mathcal{L}_{\mathrm{prop}}$. 
	
	The overall basic language $\mathcal{L}_{\nonprob}$ is the language 
	formed by closing off the formulas of $\mathcal{L}_{\mathrm{cond}}$
\footnote{Unlike \cite{ibeling}, we do not
admit the basic atoms $X$ as atoms of $\mathcal{L}$. There is no difficulty extending the semantics
to such atoms, but allowing them would needlessly complicate the proof of Theorem \ref{soundcomplete}.
}
under conjunction, disjunction, and negation.
\end{definition}	
Every formula $\alpha \in \mathcal{L}_{\mathrm{int}}$ specifies an intervention
$\mathcal{I}_{\alpha}$
by giving a list
of variables to fix and which values they are to be fixed to.
%Given $\varphi \in \mathcal{L}_{\mathrm{prop}}$,
%we say
%$\varphi' \in \mathcal{L}_{\mathrm{int}}$ is the \emph{$\mathcal{L}_{\mathrm{int}}$-equivalent}
%of $\varphi$
%% and write $\varphi \leftrightarrow \varphi'$, % this notation is never used
%if $\varphi$ is a propositionally consistent, purely conjunctive formula over literals and $\varphi'$ results from a reordering
%of literals and deletion of repeated literals
%in $\varphi$. For example, the $\mathcal{L}_{\mathrm{int}}$-equivalent of $\lnot X_2 \land X_1 \land X_1$ is $X_1 \land \lnot X_2$.
Given a \emph{subjunctive conditional} formula $\langle \alpha \rangle \beta \in \mathcal{L}_{\mathrm{cond}}$, we call $\alpha$
the \emph{antecedent} and $\beta$ the \emph{consequent}.
We use $[\alpha]$ for the dual of $\langle\alpha\rangle$,
i.e., $[ \alpha ]\beta$ abbreviates $\lnot \langle\alpha\rangle (\lnot \beta)$.
Note that $\langle\rangle \varphi$ holds in a program if the unmodified program halts with a tape making $\varphi$ true.

\subsubsection{Semantics}
The semantics of the basic language are defined from considering a subjunctive conditional to be true
in a simulation model
when
the program so intervened upon as to make its antecedent hold halts with such values of the tape variables as make
its consequent hold.
For example, consider a simple model that checks if the first memory tape square $X_0$ is $1$ and if so writes a $1$ into the second tape
square $X_1$, and otherwise simply halts. This program satisfies the formulas $\langle\rangle \lnot X_0$, $\langle\rangle \lnot X_1$,
but also the counterfactual formula $\langle X_0 \rangle (X_0 \land X_1)$: holding the first memory square fixed to $1$ causes a write of the value
$1$ into the second tape square, thus satisfying the consequent $X_0 \land X_1$.
%for both non-probabilistic simulation models
%and for the
%completely deterministic case
%when one takes a probabilistic simulation model and fixes values for all variables in the random bit tape.
Formally:
\begin{definition}
	\label{semanticsofLnonprob}
	Let $\mathsf{T}$ be a non-probabilistic simulation model. Define $\mathsf{T} \models_{\nonprob}
	\langle\alpha\rangle \beta$ iff $\mathcal{I}_{\alpha}(\mathsf{T})$ halts with a memory tape whose variable
	assignment satisfies $\beta$.
	Now suppose $\mathsf{T}$ is probabilistic, and fix values for all squares on the random bit tape to some sequence $\mathbf{r} \in \{0, 1\}^\infty$.
	Define $\mathsf{T}, \mathbf{r} \models \langle \alpha \rangle \beta$ iff $\mathcal{I}_{\alpha}(\mathsf{T})$ when run with its random bit tape fixed to $\mathbf{r}$ halts
	with a resultant memory tape satisfying $\beta$.
	Define (in both cases) satisfaction of arbitrary formulas of $\mathcal{L}_{\nonprob}$ in the familiar way by recursion.
\end{definition}
In a sense, the validities of the non-probabilistic setting carry over
to this setting, as we will now show.
For $\varphi \in \mathcal{L}_{\nonprob}$, write $\models_{\nonprob} \varphi$
if $\varphi$ is valid in the class of all non-probabilistic simulation models.
We will see that all such formulas are still valid for probabilistic simulation models, under Definition \ref{semanticsofLnonprob},
once one fixes the random bit tape to a particular sequence.
\begin{lemma}
\label{validitieslemma}
	$\models_{\nonprob} \varphi$ if and only if, for all probabilistic simulation models
$\mathsf{T}$ and all $\mathbf{r} \in \{0, 1\}^{\infty}$, we have
that $\mathsf{T}, \mathbf{r} \models \varphi$.
\end{lemma}
\begin{proof}
Suppose $\models_{\nonprob} \varphi$. Consider some probabilistic simulation model $\mathsf{T}$ and sequence
$\mathbf{r} \in \{0, 1\}^{\infty}$. $\varphi$ is composed of $\mathcal{L}_{\mathrm{cond}}$-atoms, of the form $\langle \alpha \rangle \beta$.
What is the behavior of $\mathcal{I}_{\alpha}(\mathsf{T}), \mathbf{r}$?
Either $\mathcal{I}_{\alpha}(\mathsf{T}), \mathbf{r}$ reads only a finite portion of $\mathbf{r}$ or reads an
unbounded portion of $\mathbf{r}$ (in the latter case, it also does not halt).
If only a finite portion is read, let
$N(a)$ be the maximal random bit tape square reached of $\mathbf{r}$. Let $N$ be the maximum
of the $N(a)$ for all atoms $a$ in $\varphi$, clearly existent as $\varphi$ has finite length.
Construct a Turing machine $\mathsf{T}'$ from $\mathsf{T}$ that embeds the contents of $\mathbf{r}$ up to index $N$
into its code, replacing any read from $\mathbf{r}$ with its value. This is possible in a finite amount of code
as we only
have to include values up to $N$ in $\mathsf{T}'$.

What if $\mathcal{I}_{\alpha}(\mathsf{T}), \mathbf{r}$ ends up reading an unbounded portion of $\mathbf{r}$?
We note that it is possible to write code in $\mathsf{T}'$ to check if the
machine is being run under an $\alpha$-fixing
intervention---i.e., conditional code that runs under $\mathcal{I}_{\alpha}(\mathsf{T}')$
and no other intervention.\footnote{
For the precise details of this construction, see \cite{ibeling}. Briefly, if one wants to check if
some $X_i$ is being held fixed by an intervention, one can try to toggle $X_i$; this attempt
will be successful iff $X_i$ is not currently being fixed by an intervention.
}
Add such code to $\mathsf{T}'$, including an infinite loop conditional on an $\alpha$-intervention
for each case where $\mathcal{I}_{\alpha}(\mathsf{T}), \mathbf{r}$ reads an unbounded portion of $\mathbf{r}$.
Now, for all atoms $\langle \alpha\rangle \beta$, $\mathsf{T}' \models_{\nonprob} \langle \alpha\rangle \beta$ iff $\mathsf{T}, \mathbf{r} \models \langle \alpha\rangle \beta$. % check
As this
holds for any atom of $\varphi$, and $\models_{\nonprob} \varphi$, we have that $\mathsf{T}, \mathbf{r} \models \varphi$ as desired.

Now, suppose that $\mathsf{T}, \mathbf{r} \models \varphi$ for all probabilistic $\mathsf{T}, \mathbf{r}$. We want to see that $\models_{\nonprob} \varphi$. Given a non-probabilistic $\mathsf{T}$, convert $\mathsf{T}$ to a probabilistic TM $\mathsf{T}'$ that never reads from its random tape, and take any random tape $\mathbf{r}$. Then $\mathsf{T}', \mathbf{r} \models \varphi$ so that $\mathsf{T} \models_{\nonprob} \varphi$.
\qed
\end{proof}

\subsection{Adding Probabilities}
\subsubsection{Syntax}
$\mathcal{L}$ is the language of linear inequalities over probabilities that
formulas of $\mathcal{L}_{\nonprob}$ hold. % TODO: add somewhere else this; it is in the same vein as the simpler language considered in \cite{Fagin}.
More precisely:
\begin{definition}
	Let $\mathcal{L}_{\mathrm{ineq}}$ be the language of formulas of the form
	\begin{equation}
\label{atomicweightformula}
		a_1 \mathbb{P}(\varphi_1) + \dots + a_n \mathbb{P}(\varphi_n) \le c
\end{equation}
	for some $n \in \mathbb{N}$, and $c, a_1, \dots, a_n \in \mathbb{Z}$, $\varphi_1, \dots, \varphi_n \in \mathcal{L}_{\nonprob}$.
	Then $\mathcal{L}$ is the language of propositional formulas formed by closing off
	$\mathcal{L}_{\mathrm{ineq}}$
	under conjunction, disjunction, and negation.
\end{definition}
We
sometimes write inequalities of a different form from (\ref{atomicweightformula}) with the understanding 
that they can be readily converted into some $\mathcal{L}$-formula.
For example, an inequality with a $>$ sign is a negation of a $\mathcal{L}_{\mathrm{ineq}}$-formula.

\subsubsection{Semantics}
Let $\mathsf{T}$ be a probabilistic simulation model.
We will shortly define a probability $\mathbb{P}_{\mathsf{T}} : \mathcal{L}_{\nonprob} \to [0, 1]$.
Now suppose a  
given $\varphi \in \mathcal{L}_{\mathrm{ineq}}$ 
has the form (\ref{atomicweightformula}).
Then $\mathsf{T} \models \varphi$
iff the inequality (\ref{atomicweightformula}) holds when each $\mathbb{P}(\varphi_i)$
factor takes the value $\mathbb{P}_{\mathsf{T}}(\varphi_i)$.
Satisfaction $\mathsf{T} \models \varphi$ for arbitrary $\varphi \in \mathcal{L}$
is then defined familiarly by recursion.
Given $\varphi \in \mathcal{L}_{\nonprob}$, the probability $\mathbb{P}_{\mathsf{T}}(\varphi)$
is simply the (standard) measure of the set of 
infinite bit sequences $\mathbf{r}$ for which $\mathsf{T}, \mathbf{r} \models \varphi$.
More formally: let $\Sigma$ be the $\sigma$-algebra on $\{0, 1\}^\infty$ generated by cylinder sets and 
	$\mu$ be the standard measure defined on $\Sigma$.\footnote{
	That is, as the product measure of $\mathrm{Bernoulli}(1/2)$ measures, as defined in, e.g., \cite{DBLP:books/cu/p/FreerRT14}.
	}
Now let $S(\varphi) = \{\mathbf{r} \in \{0, 1\}^\infty : \mathsf{T}, \mathbf{r} \models \varphi\}$. Then we define $\mathbb{P}_T(\varphi) = \mu(S(\varphi))$. The following Lemma
ensures that $S(\varphi)$ is always measurable, so that this definition is valid.

\begin{lemma}
	For any $\varphi \in \mathcal{L}_{\nonprob}$, we have $S(\varphi) \in \Sigma$.
\end{lemma}
\begin{proof}
Proof by induction on the structure of $\varphi$.
	If $\varphi = \lnot \psi$, then $S({\varphi})$ is the complement of a set in $\Sigma$
and hence is in $\Sigma$.
The case of a conjunction or disjunction is similar since $\Sigma$ is closed under intersection
and union.
The base case is that of the atoms.
	Consider an atom of the form $\langle \alpha \rangle \beta$. If $\mathcal{I}_{\alpha}(\mathsf{T})$ halts on $\mathbf{x}$ with random bit tape fixed to $\mathbf{r}$, then it does so reading only a finite portion of $\mathbf{r}$.
	Thus $S({\langle \alpha \rangle \beta})$ is the  union of
cylinder sets extending finite strings on which $\mathcal{I}_{\alpha}(\mathsf{T})$ halts with a result satisfying $\beta$, and hence is in $\Sigma$.
\qed
\end{proof}
This probability is \emph{coherent} in the sense that it plays well with the logic of the basic language:
\begin{proposition}
\label{coherence}
For any probabilistic $\mathsf{T}$ we have,
\begin{enumerate}
	\item $\mathbb{P}_{\mathsf{T}}(\varphi) = 1$ if $\models_{\nonprob} \varphi$ for $\varphi \in \mathcal{L}_{\nonprob}$
	\item $\mathbb{P}_{\mathsf{T}}(\varphi) \le \mathbb{P}_{\mathsf{T}}(\psi)$ whenever $\models_{\nonprob} \varphi \rightarrow \psi$ for $\varphi, \psi \in \mathcal{L}_{\nonprob}$
\item $\mathbb{P}_{\mathsf{T}}(\varphi) = \mathbb{P}_{\mathsf{T}}(\varphi \land \psi) + \mathbb{P}_{\mathsf{T}}(\varphi \land \lnot \psi)$ for all $\varphi, \psi \in \mathcal{L}_{\nonprob}$
\end{enumerate}
\end{proposition}
\begin{proof}
(1) holds since in this case, by Lemma \ref{validitieslemma}, $S({\varphi}) = \{0, 1\}^{\infty}$. (2) holds since in this case, $S({\varphi}) \subseteq S({\psi})$. Finally (3) holds by noting $\models_{\nonprob} \varphi \leftrightarrow ((\varphi \land \psi) \lor (\varphi \land \lnot \psi))$, applying (2), and noting that $S({\varphi \land \psi})$ and $S({\varphi \land \lnot \psi})$ are disjoint.
\qed
\end{proof}
A corollary of part (2) is that logical equivalents under $\models_{\nonprob}$ preserve probability.
%The following will occasionally be useful:
%\begin{proposition}
%\label{calc-prob}
%For any antecedent $\alpha$ and formula $\beta$, the following hold:
%\begin{align*}
%\mathbb{P}(\langle\alpha\rangle \beta) + \mathbb{P}([]\bot) &= \mathbb{P}([\alpha] \beta) \\
%	\mathbb{P}([] \bot) + \mathbb{P}(\langle\alpha\rangle\beta) + \mathbb{P}(\langle\alpha\rangle (\lnot \beta)) &= 1.
%\end{align*}
%\qed
%\end{proposition}

\subsection{The Case of Almost-Surely Halting Simulations}
An interesting special case
is that of the simulation models that halt almost-surely, i.e., with probability $1$ under
every intervention. Call this class $\mathcal{M}^{\downarrow}$.
Following
the urging of \cite{icard2017beyond} we have not restricted the definition of probabilistic simulation model
to such models. We will see that from a logical point of view,
this case is a natural probabilistic analogue of the class $\mathcal{M}^{\downarrow}_{\nonprob}$ of non-probabilistic
simulation 
models that halt under every intervention.
By this we mean that we may prove an analogue to Lemma \ref{validitieslemma}.
Write $\models_{\nonprob}^{\downarrow} \varphi$ if
$\varphi \in \mathcal{L}_{\nonprob}$ is valid in $\mathcal{M}^{\downarrow}_{\nonprob}$.
Note that
Lemma \ref{validitieslemma} does \emph{not} hold if one merely changes all the preconditions to be
halting/almost-surely halting:
consider
a probabilistic simulation model $\mathsf{T}$
that repeatedly reads random bits and halts at the first
$1$ it discovers; this program is almost-surely halting.
But if $\mathbf{r}$ is an infinite
sequence of $0$s, then $\mathsf{T}, \mathbf{r} \not\models \langle\rangle \top$,
even though $\models_{\nonprob}^{\downarrow} \langle\rangle \top$.
Crucially, we must move to the perspective of probability and measure
to see the analogy:
\begin{lemma}
\label{validitieslemmaashalting}
%Let $\mathcal{M}^{\downarrow}_{\mathrm{prob}}$ be the class of probabilistic simulation models for which
%the program almost surely halts on any input tape and intervention.
%For the remainder of the lemma, understand $\models$ to mean $\mathcal{L}$-validity
%in the restricted class $\mathcal{M}^{\downarrow}_{\mathrm{prob}}$.
%Then $\models \varphi$
%if and only if, for all
%$(\mathsf{T}, \mathbf{x}) \in \mathcal{M}^{\downarrow}_{\mathrm{prob}}$
%we have that $(\mathsf{T}, \mathbf{x}, \mathbf{r}) \models \varphi$ for all $\mathbf{r} \in \{0, 1\}^{\infty}$
%but on a set of measure $0$.
$\models_{\nonprob}^{\downarrow} \varphi$ if and only if, for all $\mathsf{T} \in \mathcal{M}^{\downarrow}$,
we have $\mathsf{T}, \mathbf{r} \models \varphi$ for all $\mathbf{r} \in \{0, 1\}^\infty$ except on a set of measure
$0$.
\end{lemma}
\begin{proof}
Suppose $\models_{\nonprob}^{\downarrow} \varphi$. We claim
that for all $\mathsf{T} \in \mathcal{M}^{\downarrow}$
we have $\mathsf{T}, \mathbf{r} \models \varphi$ for all $\mathbf{r}$ except
on a set of measure $0$.
Again,
consider
an atom $\langle\alpha\rangle\beta$ appearing in $\varphi$.
The set of $\mathbf{r}$ for which
$\mathcal{I}_{\alpha}(\mathsf{T}), \mathbf{r}$ does not halt
has measure $0$, given that $\mathsf{T} \in \mathcal{M}^{\downarrow}$.
On each such $\mathbf{r}$, the run of $\mathcal{I}_{\alpha}(\mathsf{T}), \mathbf{r}$ must read infinitely many bits of $\mathbf{r}$:
otherwise, the intervened machine would have a nonzero probability of not halting.
Thus, excluding such $\mathbf{r}$, it is possible to repeat the construction of $\mathsf{T}'$
from the proof of Lemma \ref{validitieslemma} for $\langle\alpha\rangle\beta$, and
in doing this construction we are already ignoring
all cases where an unbounded portion of $\mathbf{r}$ is read. 
This means that we do not have to include any infinite loops in $\mathsf{T}'$,
and $\mathsf{T}'$ will be always-halting. If we exclude all the such $\mathbf{r}$ arising
from all antecedents of atoms
of $\varphi$, then we only exclude a set of measure $0$ since there are finitely many atoms.
Except for such $\mathbf{r}$, the construction works, and
$\mathsf{T}'$ has, as before, the same behavior as $\mathsf{T}$. But since $\models_{\nonprob}^{\downarrow} \varphi$,
we have that
$\mathsf{T}, \mathbf{r} \models \varphi$ except on the excluded set of measure $0$.

For the opposite direction, let $\mathsf{T} \in \mathcal{M}^{\downarrow}_{\nonprob}$.
We wish to show that $\mathsf{T} \models_{\nonprob} \varphi$.
Convert $\mathsf{T}$ to an identical probabilistic simulation program $\mathsf{T}'$ that never
reads from its random tape. We have $\mathsf{T}', \mathbf{r} \models \varphi$
for all $\mathbf{r}$ but on a set of measure $0$; in particular, for at least one $\mathbf{r}$.
This implies $\mathsf{T} \models_{\nonprob}\varphi$.
\qed
\end{proof}

\section{Axiomatic Systems}
We will now give an axiomatic system for
reasoning in $\mathcal{L}$ and prove
that it is \emph{sound} and \emph{complete} with respect to
probabilistic simulation models: it proves
all (completeness) and only (soundness) the formulas of
$\mathcal{L}$ that hold for all probabilistic simulation models.
We will give an additional system that is sound and complete
for validities with respect to the almost-surely halting simulation models $\mathcal{M}^{\downarrow}$.
\begin{definition}
Let $\textsf{AX}$ be a
set of rules and axioms formed by combining the following three modules.
\begin{enumerate}
\item $\textsf{PC}$: propositional reasoning (tautologies and \emph{modus ponens})
over atoms of $\mathcal{L}$.
\item $\textsf{Prob}$: the following axioms:
\begin{eqnarray*}
	\textsf{NonNeg}. && \mathbb{P}(\varphi) \ge 0 \\
	\textsf{Norm}.&& \mathbb{P}(\top) = 1 \\
 \textsf{Add}. && \mathbb{P}(\varphi \land \psi) + \mathbb{P}(\varphi \land \lnot \psi) = \mathbb{P}(\varphi) \\
	\textsf{Dist}. && \mathbb{P}(\varphi) = \mathbb{P}(\psi) \mbox{ whenever } \models_{\nonprob} \varphi \leftrightarrow \psi \\
\end{eqnarray*} 
\item $\textsf{Ineq}$, an axiomatization (see \cite{Fagin}) for
reasoning about
linear inequalities:
\begin{eqnarray*}
%	\textsf{Id}. && a_1 \mathbb{P}(\varphi_1) + \dots + a_n \mathbb{P}(\varphi_n) \ge  a_1 \mathbb{P}(\varphi_1) + \dots + a_n \mathbb{P}(\varphi_n) \\
	\textsf{Zero}. && 
	(a_1 \mathbb{P}(\varphi_1) + \dots + a_n \mathbb{P}(\varphi_n) \le c)
	\\ && \Leftrightarrow
	(a_1 \mathbb{P}(\varphi_1) + \dots + a_n \mathbb{P}(\varphi_n) + 0 \mathbb{P}(\varphi_{n+1}) \le c)\\
	\textsf{Permutation}. && 
	(a_1 \mathbb{P}(\varphi_1) + \dots + a_n \mathbb{P}(\varphi_n) \le c)
	\Leftrightarrow
	(a_{j_1} \mathbb{P}(\varphi_{j_1}) + \dots + a_{j_n} \mathbb{P}(\varphi_{j_n}) \le c)\\
	&& \mbox{ when } j_1, \dots, j_n \mbox{ are a permutation of } 1, \dots, n\\
	\textsf{AddIneq}. &&
	(a_1 \mathbb{P}(\varphi_1) + \dots + a_n \mathbb{P}(\varphi_n) \le c) \land
	(a'_1 \mathbb{P}(\varphi_1) + \dots + a'_n \mathbb{P}(\varphi_n) \le c')\\
	&& \Rightarrow
	((a_1 + a'_1) \mathbb{P}(\varphi_1) + \dots + (a_n + a'_n) \mathbb{P}(\varphi_n) \le (c + c')) \\
	\textsf{Mult}. &&
	(a_1 \mathbb{P}(\varphi_1) + \dots + a_n \mathbb{P}(\varphi_n) \le c)\\
	&& \Rightarrow
	(b a_1 \mathbb{P}(\varphi_1) + \dots + b a_n \mathbb{P}(\varphi_n) \le b c)
	\mbox{ for any } b > 0 \\
	\textsf{Dichotomy}. && (a_1 \mathbb{P}(\varphi_1) + \dots + a_n \mathbb{P}(\varphi_n) \le c) \lor (a_1 \mathbb{P}(\varphi_1) + \dots + a_n \mathbb{P}(\varphi_n) \ge c)\\
	\textsf{Mono}. && (a_1 \mathbb{P}(\varphi_1) + \dots + a_n \mathbb{P}(\varphi_n) \le c) \\
	&& \Rightarrow (a_1 \mathbb{P}(\varphi_1) + \dots + a_n \mathbb{P}(\varphi_n) < b) \mbox{ if } b > c
\end{eqnarray*}

\end{enumerate}
Additionally, let $\textsf{AX}^{\downarrow}$ be the system formed
in exactly the same way, but
	replacing $\models_{\nonprob}$ with $\models_{\nonprob}^{\downarrow}$. 
\end{definition}
Note that the non-probabilistic validities $\models_{\nonprob}$ and $\models_{\nonprob}^{\downarrow}$, appearing in $\textsf{Dist}$,
have been completely axiomatized in \cite{ibeling}.
The main result is:
\begin{theorem}
\label{soundcomplete}
	$\textsf{AX}$ (respectively, $\textsf{AX}^{\downarrow}$) is sound and complete for the
	validities of $\mathcal{L}$ with respect to $\mathcal{M}$ (respectively, $\mathcal{M}^{\downarrow}$).
\end{theorem}
\begin{proof}
Soundness (of $\textsf{Prob}$) follows from Lemma \ref{validitieslemma}, Proposition \ref{coherence}, and, for the almost-surely
halting case, Lemma \ref{validitieslemmaashalting}.
For completeness,
consider the general case of $\mathcal{M}$ first.
As usual, it suffices to show
that any consistent $\varphi \in \mathcal{L}$
is satisfiable by some probabilistic simulation model.
We put $\varphi$ into a normal form from which we construct a canonical model.
By
$\textsf{PC}$
we
may
suppose $\varphi$ is in disjunctive normal form. We may further
suppose
that it is a conjunction of $\mathcal{L}_{\mathrm{ineq}}$-literals, as
at least one (conjunctive) clause in the disjunctive normal form must be consistent.
Let $a_1, \dots, a_n \in \mathcal{L}_{\mathrm{cond}}$ be the atoms that appear
inside any probability $\mathbb{P}$ in $\varphi$, and
let $\delta_{1}, \dots, \delta_{2^n}$ represent all
the formulas of the form $l_1 \land \dots \land l_n$
that can be obtained by setting each $l_i$ to either $a_i$ or $\lnot a_i$.
We then have the following, which is a kind of normal form result:
\begin{lemma}[Lemma 2.3, \cite{Fagin}]
\label{bigconjunction}
$\varphi$ is provably-in-$\textsf{AX}$ equivalent to a conjunction
\begin{align}
\label{normalform}
(\mathbb{P}(\delta_1) \ge 0) \land \dots \land (\mathbb{P}(\delta_{2^n}) \ge 0) &\land \nonumber\\
(\mathbb{P}(\delta_1) + \dots + \mathbb{P}(\delta_{2^n}) = 1) &\land\nonumber\\
(a_{1, 1} \mathbb{P}(\delta_1) + \dots + a_{1, 2^n} \mathbb{P}(\delta_{2^n}) \le c_1) &\land\nonumber\\
\dots &\land \nonumber\\
(a_{m, 1} \mathbb{P}(\delta_1) + \dots + a_{m, 2^n} \mathbb{P}(\delta_{2^n}) \le c_m) &\land\nonumber\\
(a'_{1, 1} \mathbb{P}(\delta_1) + \dots + a'_{1, 2^n} \mathbb{P}(\delta_{2^n}) > c'_1) &\land\nonumber\\
\dots &\land \nonumber\\
(a'_{m', 1} \mathbb{P}(\delta_1) + \dots + a'_{m', 2^n} \mathbb{P}(\delta_{2^n}) > c'_{m'})
\end{align}
for some integer coefficients $c_1,\dots, c_m,c_1,\dots, c'_{m'}, a_{1,1}, \dots, a_{m, 2^n}, a'_{1,1}, \dots, a'_{m', 2^n}$.
\end{lemma}
\begin{proof}
Let $\psi \in \mathcal{L}_{\nonprob}$ be any of the formulas appearing inside of a probability $\mathbb{P}$ in
$\varphi$.
Note that $\mathbb{P}(\psi) = \mathbb{P}(\psi \land l_1) + \mathbb{P}(\psi \land \lnot l_1)$ by
$\textsf{Add}$. Moving on to $l_2$, we have, provably, $\mathbb{P}(\psi \land l_1) = \mathbb{P}(\psi \land l_1 \land l_2)
+ \mathbb{P}(\psi \land l_1 \land \lnot l_2)$, and we may rewrite $\mathbb{P}(\psi \land \lnot l_1)$ similarly.
Applying this process successively, we have $\mathbb{P}(\psi) = \mathbb{P}(\psi \land \delta_1) +
\dots + \mathbb{P}(\psi \land \delta_{2^n})$. For any term in the right-hand side of this inequality,
if $\psi \Rightarrow \delta_i$, propositional reasoning by $\textsf{Dist}$ allows us to replace the term by
$\mathbb{P}(\delta_i)$, and if not, by $0$.
Thus we always have that 
$\mathbb{P}(\psi) = b_1 \mathbb{P}(\delta_1) + \dots + b_{2^n} \mathbb{P}(\delta_{2^n})$
for some coefficients $b_i$.
Applying this process to each $\mathbb{P}$-term in $\varphi$ and using
$\textsf{Ineq}$ to rewrite the left-hand sides of the inequalities, and conjoining
the (clearly provable) clauses that $\mathbb{P}(\delta_i) \ge 0$ for all $1 \le i \le 2^n$,
and $\mathbb{P}(\delta_1) + \dots + \mathbb{P}(\delta_{2^n}) = 1$, we obtain (\ref{normalform}).
\qed
\end{proof} 

The conjunction (\ref{normalform})
can be seen as a system of simultaneous inequalities over $2^n$ unknowns,
$\mathbb{P}(\delta_1), \dots, \mathbb{P}(\delta_{2^n})$.
$\textsf{Ineq}$ is actually sound and complete for such systems (we refer the
reader to Section 4 of \cite{Fagin} for the proof of this fact). So if $\varphi$ is consistent with $\textsf{AX}$---which includes
$\textsf{Ineq}$---this system must have a solution.
Thus there are values $\mathbb{P}(\delta_i)$ solving (\ref{normalform}).
We will now construct a probabilistic simulation model having precisely these probabilities
of satisfying each $\delta_i$.
Note that for any $\delta_i$ with $\models_{\nonprob} \bot \leftrightarrow \delta_i$
it is provable that $\mathbb{P}(\delta_i) = 0$, and we may conjoin this to
(\ref{normalform}). 
%This is because we can use
%$\textsf{AX}_{\mathrm{prob}}$, specifically the $\textsf{Prob}$ axioms,
%to deduce that its probability is $0$, and conjoin a conjunct $\mathbb{P}(\delta_i) \le 0$
%to (\ref{normalform}). (It is crucial that we can actually deduce probability $0$, 
%as this means we can append this conjunct to (\ref{normalform}) while
%preserving $\textsf{AX}_{\mathrm{prob}}$-consistency.) This forces any 
%solution to have $\mathbb{P}(\delta_i) = 0$.
Note also that $\delta_i \land \delta_j$
is unsatisfiable for any $i \neq j$.
Given these two observations,
the following Lemma
implies the result.
%If one wants to complete the analogy to other completeness proofs
%that has been drawn thus far, then note that
%the requisite $p_i$ values play the role of a selection function.
\begin{lemma}
\label{makeprobshold}
For any collection of satisfiable
$\mathcal{L}_{\nonprob}$-formulas $\varphi_1, \dots, \varphi_n$
no two of which are jointly satisfiable, and
any
rational probabilities
$p_1, \dots, p_n \ge 0$
such that $p_1+\dots+p_n = 1$, there is a probabilistic simulation model
$\mathsf{T}$ such that $\mathbb{P}_{\mathsf{T}}(\varphi_i) = p_i$
for all $i$, $1 \le i \le n$.
\end{lemma}
\begin{proof}
Since the $\varphi_i$ are satisfiable, there are non-probabilistic
simulation models $\mathsf{T}_{\nonprob, 1}, \dots \mathsf{T}_{\nonprob, n}$
such that for all $i = 1, \dots, n$, we have $\mathsf{T}_{\nonprob, i} \models_{\nonprob} \varphi_i$.
Further, we may suppose the machines so constructed use only a bounded number of memory tape squares.\footnote{
Why? Since $\varphi_i$ are satisfiable, they are consistent with the axiomatization for non-probabilistic
simulation models given by \cite{ibeling}, and hence are satisfied by the canonical models given in \cite{ibeling}.
These models use only boundedly many tape squares.
}
Thus let the maximum index of a tape square used by any of the $\mathsf{T}_{\nonprob, i}$ be $N$.
We now describe $\mathsf{T}$ informally.
Suppose without loss of generality that for all $i$, $p_i = a_i / b$ for some common denominator $b$.
Let $\mathsf{T}$ draw a random number $r$ from $1$ up to $b$ uniformly, and ensure that $\mathsf{T}$ does any auxiliary
computations it might need only on squares
with indices at least $N+1$.
Check whether $r \le a_1$, and if so,
let $\mathsf{T}$ branch into the code of $\mathsf{T}_{\nonprob, 1}$.
If not, check if $a_1+1 \le r \le a_1 + a_2$
and if so, branch into $\mathsf{T}_{\nonprob, 2}$.
Repeat the process for $p_3, \dots, p_n$.
It's clear that the probability
of branching into each $\mathsf{T}_{\nonprob, i}$ block is exactly $p_i$,
and the same is true
under any relevant (i.e., involving only memory tape variables that appear in one of the $\varphi_i$) intervention
on $\mathsf{T}$:
we may suppose
any auxiliary computations
$\mathsf{T}$ might require use only memory tape squares with indices past $N$.
After branching into the $i$th block, the behavior
of $\mathsf{T}$ is
exactly
the same as that of $\mathsf{T}_{\nonprob, i}$, meaning that any random bit tape fixings
that end up causing a branch into this block will belong to $S(\varphi_i)$.
Another random bit tape fixing that causes a branch into another block, say the $j$th,
cannot belong to $S(\varphi_i)$ since 
$\varphi_i, \varphi_j$ are jointly unsatisfiable.
Thus, $\mathbb{P}_{\mathsf{T}}(\phi_i) = p_i$
for all $i$.
\qed
\end{proof}
Finally, we must see that this model lies in $\mathcal{M}^{\downarrow}$
if the original formula is consistent with $\textsf{AX}^{\downarrow}$.
\cite{ibeling} has shown that $\models_{\nonprob}^{\downarrow} [\alpha]\beta \rightarrow \langle\alpha\rangle\beta$. % TODO: cleanup
Then in the proof of Lemma \ref{makeprobshold}, we may suppose
that each $\mathsf{T}_{\nonprob, i}$ block contains only always-halting code,\footnote{
Since the canonical programs of \cite{ibeling} for $\mathcal{M}^{\downarrow}_{\nonprob}$ contain only such code.
} and
hence that $\mathsf{T}$ does not contain any loops either: thus it almost-surely halts.
\qed
\end{proof}

\section{Computational Complexity}
Call the problem of deciding if a formula $\varphi \in \mathcal{L}$
is satisfiable $\textsc{Prob-Sim-Sat}(\varphi)$. Theorem \ref{probsimsat} shows
that solving this problem is no more complex than is propositional satisfiability.
\begin{theorem}
\label{probsimsat}
$\textsc{Prob-Sim-Sat}(\varphi)$ is $\mathsf{NP}$-complete in $|\varphi|$ (where this length
is computed standardly).
\end{theorem}
\begin{proof}
It's $\mathsf{NP}$-hard since, given any propositional $\pi$, the formula
$\mathbb{P}(\langle\rangle \pi) > 0$ is satisfiable iff $\pi$ is satisfiable
(consider a machine that does nothing but write a satisfying memory tape assignment out).
In order to show that the satisfiability problem is in $\mathsf{NP}$,
we
give the following nondeterministic satisfiability algorithm: guess a program
from a class of programs (that we will define shortly) that includes the program constructed
in Lemma \ref{makeprobshold}
---call this canonical program $\mathsf{T}_{\varphi}$---and check (in polynomial time) if
it satisfies $\varphi$. This algorithm decides satisfiability since, by soundness, a satisfiable formula
must be consistent, and hence has a canonical model of the form constructed in
Lemma \ref{makeprobshold}. For the remainder of the proof, by the ``length of a number,''
we just mean the length of its computer (binary) representation. The ``length of a rational''
is the sum of the lengths of its numerator and its denominator.

What is the class of probabilistic simulation models that we may limit our guesses to?
For some fixed constants $C, D \in \mathbb{N}$,
we will define a class $\mathcal{M}_{\varphi, C, D}$.
We will then show that there exist $C, D$ such that the canonical program of
Lemma \ref{makeprobshold} belongs to $\mathcal{M}_{\varphi, C, D}$ for all consistent $\varphi$.
Let $\mathcal{M}_{\varphi, C, D}$
be
the fragment of probabilistic simulation models
whose code consists of the following:
\begin{enumerate}
\item Code to draw a random number uniformly between $1$ and some $N$, such that $N$ has length at most $D |\varphi|^3$.
\item At most $n = C | \varphi|$ \emph{branches}, that is, copies of: an if-statement %TODO: add some note about programming
with condition $\ell \le r \le u$,
whose body is a canonical program $\mathsf{P}_{\psi_i}$ for some $\psi_i \in \mathcal{L}_{\nonprob}$, of
the same form as the non-probabilistic canonical models (i.e., in the class defined in the proof of Theorem 2 from \cite{ibeling}).
\end{enumerate}
Letting $\ell_i, u_i$ be the bounds for the $i$th copy in (2), we also require that $\ell_1 = 1$,
and that $\ell_{i+1} = u_i + 1$ for all $i$, and that $u_n = N$. 
The following fact from linear algebra (we refer the reader to \cite{Fagin} for the proof)
helps us to show that for all consistent $\varphi$, the canonical program $\mathsf{T}_{\varphi}$
belongs to $\mathcal{M}_{\varphi, C, D}$ for some $C, D$.
\begin{lemma}
\label{limitsolutions}
A system of $m$ linear inequalities 
with integer coefficients of length at most $\ell$
that has a nonnegative solution has
a nonnegative solution with at most $m$ variables nonzero, and where
the variables have length at most $\mathcal{O}(m \ell + m \log m)$.
\qed
\end{lemma}
Apply this lemma to (\ref{normalform}). Each inequality in (\ref{normalform}) originally
came from $\varphi$, so there are $\mathcal{O}(|\varphi|)$ of them. Further, recall
that
each
integer
coefficient in (\ref{normalform}) came from summing up a subset of $2^n$ coefficients
originally from $\varphi$, with $n$ is the number of atoms appearing anywhere
inside $\mathbb{P}$ expressions in $\varphi$.
As this $n$ is thus $\mathcal{O}(|\varphi|)$---and hence $2^n$ is $\mathcal{O}(|\varphi|)$ \emph{in length}---and
each  original coefficient is also $\mathcal{O}(|\varphi|)$ in length, each coefficient is $\mathcal{O}(|\varphi|)$ in length as well
(lengths of products add).
Thus Lemma \ref{limitsolutions} shows that without loss of generality, we may
suppose that the solutions for the $\mathbb{P}(\delta_i)$ of (\ref{normalform}) have
$\mathcal{O}(|\varphi|^2)$ length. The common denominator of these $\mathcal{O}(|\varphi|)$
rationals hence has
$\mathcal{O}(|\varphi|^3)$ length. The construction of Lemma \ref{makeprobshold}
has one branch for each of them, and hence $\mathcal{O}(|\varphi|)$ branches.
This shows the existence of $D$ for part (1) of the definition of $\mathcal{M}_{\varphi, C, D}$
and the existence of a $C$ for part (2). We will abbreviate
$\mathcal{M}_{\varphi} = \mathcal{M}_{\varphi, C, D}$
for some choice of $C, D$ thus guaranteed.

It remains to show that given any program $\mathsf{T} \in \mathcal{M}_{\varphi}$,
we can check if $\mathsf{T} \models \varphi$ in polynomial time. It suffices to show that
checking if $\mathsf{T} \models \psi$ for $\psi \in \mathcal{L}_{\mathrm{ineq}}$ is polynomial
time: if we know whether $\mathsf{T} \models \psi$
for every $\psi$ that $\varphi$ is built out of, we can decide in linear time if $\mathsf{T} \models \varphi$.
Thus suppose $\psi$ has the form $a_1 \mathbb{P}(\varphi_1) + \dots +
a_n \mathbb{P}(\varphi_n) \le c$.
\cite{ibeling} shows that one may check if the $\mathsf{P}_{\psi_i}$ in part (2) of the definition of $\mathcal{M}_{\varphi}$
satisfy any formula of the basic language
$\mathcal{L}_{\nonprob}$ in polynomial time.
Then we can easily compute $\mathbb{P}(\varphi_i)$
as simply the sum of the probabilities of each branch that satisfies $\varphi_i$.
Doing the arithmetic to check if $\psi$ is satisfied is then certainly polynomial time, so we have our result.
\qed
\end{proof}

\section{Conclusion and Future Work}
We have defined and obtained foundational results concerning a very natural extension of counterfactual intervention
on simulation models to the probabilistic case.

One critical operation in probability is \emph{conditioning}, or updating probabilities
given that some event is known to have occurred (in the subjective interpretation, updating
a belief for known information). One may already define conditional probabilities in the usual way
in the current framework, and our framework (without interventions) covers the
\emph{conditional simulation} approach to certain aspects of common-sense reasoning of \cite{DBLP:books/cu/p/FreerRT14}.
In this approach, one limits oneself to the runs satisfying a certain query; the framework considered
here would be equivalent for any queries expressible as formulas of $\mathcal{L}_{\nonprob}$.
\cite{Fagin} also give a logic for reasoning about conditional probabilities. Future work would involve extending
this system to probabilistic simulation models and studying the complexity of reasoning in that setting.

As \cite{AC17,ibeling} note, the simulation model approach invalidates many
important logical principles that are valid in other approaches \cite{Halpern2000,Pearl2009,lewis73},
such as \emph{cautious monotonicity}:
$[A] (B \land C) \rightarrow [A \land B] C$.
However the approach is otherwise quite general, and
an important future direction would be to identify and characterize subclasses of of simulation models that
validate this and other similar logical principles. We have begun investigating this extension.
An interesting consequence it
has is on
the comparison of conditional probability with the probabilities of subjunctive conditionals: while these
two probabilities are not in general equal in the classes $\mathcal{M}$ or $\mathcal{M}^{\downarrow}$, they are equal in certain
restricted classes.

A final direction we want to mention concerns ``open-world'' reasoning
including first-order reasoning
about models with some domain, where counterfactual antecedents might alter how
many individuals are being considered or which individuals fall under a property or bear certain
relations to each other.
Recursion and the tools of logic programming \cite{Goodman2014,BLOG} make this very
natural for the simulation model approach, and we would like to understand
the first- and higher-order conditional logics that result in this approach, in both the non-probabilistic
and probabilistic cases. We have also begun exploring this direction.

\bibliographystyle{splncs04}
\bibliography{ms}
\end{document}